\documentclass{article}

\usepackage{microtype}
\usepackage{graphicx}
\usepackage{subfigure}
\usepackage{booktabs} 

\usepackage{hyperref}

\usepackage[accepted]{icml2023}

\usepackage{amsmath}
\usepackage{amssymb}
\usepackage{mathtools}
\usepackage{amsthm}

\usepackage[capitalize,noabbrev]{cleveref}

\usepackage[utf8]{inputenc}
\usepackage[T1]{fontenc}    
\usepackage{hyperref}       
\usepackage{url}            
\usepackage{booktabs}       
\usepackage{amsfonts}       
\usepackage{nicefrac}       
\usepackage{microtype}      
\usepackage{xcolor}         
\usepackage{graphicx}
\usepackage[normalem]{ulem}
\usepackage{xcolor}

\usepackage{amsmath,amsfonts,bm}

\def\eqref#1{equation~\ref{#1}}

\def\1{\bm{1}}

\DeclareMathAlphabet{\mathsfit}{\encodingdefault}{\sfdefault}{m}{sl}
\SetMathAlphabet{\mathsfit}{bold}{\encodingdefault}{\sfdefault}{bx}{n}

\newcommand{\R}{\mathbb{R}}

\theoremstyle{plain}
\newtheorem{theorem}{Theorem}[section]

\theoremstyle{definition}
\newtheorem{definition}[theorem]{Definition}

\theoremstyle{remark}
\newtheorem{remark}[theorem]{Remark}

\usepackage[textsize=tiny]{todonotes}

\begin{document}

\twocolumn[
\icmltitle{
Morse Neural Networks for Uncertainty Quantification
}

\icmlsetsymbol{equal}{*}

\begin{icmlauthorlist}
\icmlauthor{Benoit Dherin}{equal,google}
\icmlauthor{Huiyi Hu}{equal,deepmind}
\icmlauthor{Jie Ren}{deepmind}
\icmlauthor{Michael W. Dusenberry}{deepmind}
\icmlauthor{Balaji Lakshminarayanan}{deepmind}
\end{icmlauthorlist}

\icmlaffiliation{google}{Google}
\icmlaffiliation{deepmind}{Google DeepMind}

\icmlcorrespondingauthor{Benoit Dherin}{dherin@google.com}

\icmlkeywords{Machine Learning, ICML, uncertainty quantification, out-of-distribution detection, anomaly detection, generative probabilistic model, deep generative models}

\vskip 0.3in
]

\printAffiliationsAndNotice{\icmlEqualContribution} 

\begin{abstract}
We introduce a new deep generative model useful for uncertainty quantification: the Morse neural network, which generalizes the unnormalized Gaussian densities to have modes of high-dimensional submanifolds instead of just discrete points. Fitting the Morse neural network via a KL-divergence loss yields 1) a (unnormalized) generative density, 2) an OOD detector, 3) a calibration temperature, 4) a generative sampler, along with in the supervised case 5) a distance aware-classifier. The Morse network can be used on top of a pre-trained network to bring distance-aware calibration w.r.t the training data. Because of its versatility, the Morse neural networks unifies many techniques: e.g., the Entropic Out-of-Distribution Detector of \cite{macedo2021entropic} in OOD detection, the one class Deep Support Vector Description method of \cite{ruff2018deeponeclass} in anomaly detection, or the Contrastive One Class classifier in continuous learning \cite{sun2021iloc}.
The Morse neural network has connections to support vector machines, kernel methods, and Morse theory in topology.

\end{abstract}

\section{Introduction}
\label{sec:intro}
Neural networks are becoming prevalent in a large number of applications both in industry and research \cite{larson-etal-2019-evaluation,sharma2022deep, Jumper2021HighlyAP,ren2019likelihood,han2021reliable,kivlichan2021measuring}. Because of their impressive performances, these models are likely to become increasingly trusted in a wider range of domains, including sensitive applications like in the medical field \cite{Roy2021doesyourdermatology,kivlichan2021measuring,Jumper2021HighlyAP,han2021reliable}. This makes the ability to quantify when the predictions of a neural network should be considered uncertain a critical issue, especially since neural networks are known to deliver wrong predictions very confidently \cite{nguyen2014easilyfooled, goodfellow2015explaining,ovadia2019can,guo2017oncalibration,hendrycks2017a,lakshminarayanan2017simple}. As a result, the development of methods to quantify neural network uncertainty is an increasingly important subject in deep learning research \cite{amodei2016concrete}.   
In particular, neural networks tend to produce confidently wrong predictions when presented with Out-Of-Distribution (OOD) inputs, that is, inputs that are far away from the data distribution with which the model was trained \cite{pml2Book, nagarajan2021understanding, liu2020simple, liu2022simple}. Detecting OOD inputs as well as devising models that are aware of the distance from inputs to the training distribution are becoming key challenges in uncertainty quantification \cite{pml2Book}.  

One classical approach to detect OOD data is to fit a generative probability density to the In-Distribution data (IND) (since OOD points are often rare) and use it to detect OOD points as points with low probability \cite{pml2Book}. This works well for normal data with a single mode, but becomes computationally prohibitive when the data has very complex modes requiring fitting large mixtures of simpler parametric models. 
On the other hand, standard deep generative models are easier to train and in theory able to express very complex modes. However, they have been shown to have some difficulty to distinguish IND from OOD even in the simpler case of distinguishing between MNIST and FashionMNIST \cite{nalisnick2018do}. Non-generative deep learning approaches like the one class deep Support Vector Data Description (SVDD) of \cite{ruff2018deeponeclass} have yielded better results on that front. 

Another approach is to leverage the supervised information from a classifier to obtain finer-grained OOD detectors. For instance, \cite{kimin2018AsimpleUnified,ren2021SimpleFix} fit a multivariate Gaussian to the classifier embeddings for each of the labels, yielding a squared distance, the Mahalanobis distance, from the Gaussian modes, creating a powerful OOD detector. The Spectral Normalized Gaussian Process (SNGP) in \cite{liu2020simple,liu2022simple} on the other hand fits an approximate Gaussian process to the classifier embeddings producing an input-dependent temperature that is used to scale the classifier logits so that it becomes distance-aware (i.e. its uncertainty grows as we move away from the training set distribution). Finally, the entropic OOD detector from \cite{macedo2021entropic} learns a classifier whose logits are distance-aware by construction and uses the entropy score of that classifier as an OOD detector. 
For further reference, we refer the reader to \cite{salehi2022a, Bulusu2020AnomalousED,ruff2018deeponeclass} for comprehensive surveys of the literature.

We propose a new deep generative model, the Morse network, which unifies a number of the separate supervised and unsupervised techniques mentioned. This model produces a join (unnormalized) density $\mu(x,y)$ by taking the kernel similarity $K(\phi_\theta(x), T(y))$ between the image of the feature $x$ and the one-hot-encoded $T(y)$ version of the label $y$, which we set to a fixed value $a$ in the unsupervised case (see sections \ref{section:unsupervised} and \ref{section:supervised} for details). The unsupervised formulation comprises mixtures of (unnormalized) Gaussian densities (and more generally exponential family densities) along with more flexible densities whose modes are submanifolds rather than discrete points.
The unsupervised Morse network with a Gaussian kernel degenerates to the deep one class SVDD proposed in \cite{ruff2018deeponeclass}, except for a built-in regularizer in the loss somewhat reminiscent of the mixup regularizer \cite{pinto2022using}. For the Cauchy kernel, it has a temperature whose form coincides with that of SNGP, except that the variance is learned and not computed by large matrix inversion. The supervised Morse network with a Laplace kernel produces a distance-aware classifier that coincides with the entropy OOD classifier from \cite{macedo2021entropic}.
Although this is not the focus of this work, we explain how the Morse network yields a sample generator by following certain gradient flows from random initial points, very much in the spirit of the Poisson generative model \cite{xu2022poisson}.

The focus of this work is to expose this unifying idea, and to explore the relationships with known approaches. Comprehensive evaluation of the approach and its extensions is the topic of future work.

\section{The Unsupervised Morse Neural Network}
\label{section:unsupervised}

We now introduce the (unsupervised) Morse neural networks. These networks produce (unnormalized) generative densities $\mu_\theta(x)\in [0, 1]$ directly on a feature space $X=\R^d$ or on a space of embeddings  $\mu_\theta(h(x))\in [0, 1]$ of the original features obtained from a pre-trained network $h(x)$. Morse neural networks are very expressive in terms of the modes they can produce (see examples below). Recall that the modes of an (unormalized) density $\mu: X\rightarrow [0,1]$ is the subset $\textrm{modes}(\mu) \subset X$ where the density achieves its highest possible value, namely 1. This set can be reduced to a single point (e.g., the mean of a Gaussian) or it can be more complex such as a smooth subset of $X$ of higher dimension, like a curve, or a surface, or, more generally a $k$-dimensional submanifold of $X$, as is the case for the Morse neural network densities.
Intuitively, these generative densities are uniformly $1$ on their mode submanifolds and decrease to $0$ as we move away from these modes at a speed controlled by a special type of kernels, which we call Morse kernels:

\begin{definition}\label{definition:morse_kernel} 
A \emph{Morse kernel} $K$ on a space $Z=\R^k$ is a positive kernel $K(z_1, z_2)$ taking its values in the interval $[0,1]$ and such that $K(z_1,z_2) = 1$ if and only if $z_1 = z_2$.
\end{definition}

Many common kernels are Morse kernels. Namely, all kernels of the form $K(z_1, z_2) = \exp(-\lambda D(z_1, z_2))$ where $D$ is a divergence in the sense of \citet{AmariShun-ichi2016IGaI} are Morse kernels (since $D(z_1, z_2) = 0$ if and only if $z_1 = z_2$). The Gaussian kernel and the Laplace Kernel are Morse Kernel, as well as the Cauchy kernel  $K(z_1, z_2) = \frac 1{1 + \lambda \|z_1 - z_2\|^2}$.

We are now ready to define the Morse neural network:

\begin{definition}
A \emph{Morse neural network} is defined by the data of 1) a neural network $\phi_\theta:X\rightarrow Z$ (with parameters $\theta$) from the feature space $X=\R^d$ to a space $Z=\R^k$, and 2) a Morse kernel $K$ on $Z$.  The (unnormalized) density of a point $x\in X$ is given by
\begin{equation}
    \mu_\theta(x) = K(\phi_\theta(x), a)
\end{equation}
where $a$ is treated as an hyper-parameter of the model.
\end{definition}
From the properties of Morse kernels, it is easy to see that $\mu_\theta(x)\in [0, 1]$ and that the modes of $\mu_\theta(x)$ (i.e., the points where $\mu_\theta(x)$ reaches 1, its highest possible value) coincide with the level set of $\phi_\theta$ (Sec. \ref{section:morse}):
\begin{equation}
    \textrm{modes}(\mu_\theta) = \{x\in X:\: \phi_\theta(x) = a\}.
\end{equation}

Fitting the Morse network to a dataset so that its modes  approximate the modes of the data distribution is explained in section \ref{section:fitting}. Applications of a fitted unsupervised Morse network are detailed in section \ref{section:applications}, which we briefly summarize here:

First, the input dependent temperature $T_\theta(x) = 1/\mu_\theta(x)$, which is 1 on the density modes and grows away from them  can be used to scale the logit of a classifier to make it distance-aware in the spirit of \cite{liu2022simple}.
Second, the classifier $s_\theta(x)= 1 - \mu_\theta(x)$ is a epistemic uncertainty score measuring our uncertainty in whether the point $x$ comes from the training distribution. Hence it can be used as an OOD detector.
Third, the function $V_\theta(x) = -\log \mu_\theta(x)$ is a form of square distance from the density modes (see section \ref{section:morse} for details and relationship with Morse-Bott theory).
In consequence, its negative gradient field flow $-\nabla_x V_\theta(x)$ converges to the mode submanifolds, giving us a way to generate new samples from random initial points very much in the spirit of the generative Poisson flow from \cite{xu2022poisson}.

We highlight several examples to showcase the flexibility of this model.

\paragraph{Location/Scale densities: $k=d$.} All the standard location/scale densities of the form $f((x-\mu)^T\Sigma^{-1}(x-\mu))$ can be recovered using a linear neural network with one layer with invertible weight matrix.  This  encompasses the multivariate Gaussian, Student-$t$, Cauchy, and Laplace densities. For all these densities we can take the same neural network  $\phi(x) = \Sigma^{-\frac 12}(\mu - x)$ and $a=0$, but we change the kernel: The Gaussian kernel $K(x, x') = \exp(-\frac 12 \|x - x'\|^2)$ produces the (unnormalized) multivariate Gaussian;  the Laplace kernel $K(x, x') = \exp(- \|x - x'\|)$ produces the (unnormalized) multivariate Laplace density; the Student-$t$ kernel with $\nu$ degrees of freedom $K(x, x') = (1 + \frac 1\nu \|x - x'\|^2)^{-\frac{d+\nu}{2} }$ produces the multivariate (unnormalized) Student-$t$ density, of which the Cauchy density is a particular case.

\paragraph{Distributions with mode submanifold: $k<d$.}
To showcase that the Morse network can produce densities with mode submanifolds, we device here an example where the density modes consist in a sphere whose radius is controlled by the hyper-parameter $a$. For that, consider the map $\phi(x) = \|x\|$ and the Gaussian kernel $K_\sigma(z, z')  = \exp(-\frac 1{2\sigma^2}(z-z')^2)$ on $\R$ with bandwidth $\sigma^2$. We  obtain the density $\mu_{a, \sigma^2}(x) = \exp(-\frac {1}{2\sigma^2}(\|x\| - a)^2)$, which has for modes the sphere of radius $a$. More generally, any regular value $a\in Z=\R^k$ of a differentiable map $\phi:\R^d\rightarrow \R^k$ will produce a density with mode submanifold of dimension $d-k$. In the kernel bandwidth limit $\sigma^2\rightarrow 0$, the resulting limiting density is the uniform density on the level sets of $\phi$.

\paragraph{Mixture models: $k>d$}
We illustrate how the Morse neural network can produce density mixtures. For instance,
a mixture of $l$ Gaussian densities on $X$ is captured by taking the neural network $\phi: X \rightarrow X^l$, where $\phi_i(x) = \Sigma^{-\frac 12}(\mu_i - x)$ is the map for the $i^{th}$ multivariate Gaussian on $X$ and $a$ is the zero vector. The kernel on $Z=X^l$ is given by a convex sum of the Gaussian kernels on the components of $Z$: $K(z, z') = \sum_{i=1}^l \alpha_i \exp(-\frac 12\|x_i - x_i'\|^2)$, where $z = (x_1, \dots, x_l)$. Since in this case, the dimension of $Z$ is larger than that of $X$ the pre-image of zero is the empty set. However, the zero-level sets of the components of $\phi$ give back the modes of the mixture of Gaussian densities. One can obtain mixtures of different densities by changing the component kernels.

\paragraph{Relation to divergences and the exponential family.}
Given a divergence $D(z,z')$ in the sense of \citet{AmariShun-ichi2016IGaI}, then $K(z,z') = e^{-\lambda D(z,z')}$ is a Morse kernel. A convex function $A$ on $Z$ produces a Bregman divergence $D_A(z,z') = A(z) + A^*(\eta') - z\eta'$ where $A^*$ is the dual convex function and $\eta' = \nabla A(z')$ is the Legendre transform (see \citet{AmariShun-ichi2016IGaI} for details). The Morse network for this kernel is then 
$
\mu_\theta(x) = e^{
\lambda(
    \phi_\theta(x)a - A^*(a) - A(\phi_\theta(x))
)
}
$
which an unnormalized version of the exponential family with canonical parameter $a$, cumulant generating function $A^*$, and dispersion $\lambda$, when $\phi_\theta(x)$ is the identity. One recovers the Gaussian case with $A(z)=\frac 12 z^2$.

\paragraph{Relation to Energy-based models.}
Energy-based models \cite{goodfellow2016deep} have been commonly used in generative modeling \cite{zhao2017energybased, gao2021learning, arbel2021generalized,che2010yougan}. In fact, the Morse neural network can be rewritten as an unormalized energy-based model as follows
\begin{equation}
\mu_\theta(x) = e^{-V_\theta(x)}
\end{equation}
where the model energy is parameterized using the Morse kernel $V_\theta(x) = -\log K(\phi_\theta(x), a)$. In section \ref{section:morse}, we show that this positive function satisfies the Morse-Bott condition on the density modes and thus can be interpreted as a sort of squared distance from the modes. Note that normalized energy models have been used in the context of uncertainty quantification \cite{wang2021energybased}, where a classifier uncertainty is added as extra dimension and learned as joint energy model. Also note that the Morse neural network with exponential family Morse kernel resembles the conjugate energy models from \citet{wu2021conjugate}.

\subsection{Applications and experimental results}\label{section:applications}

\paragraph{OOD detection:}  Since $\mu_\theta(x) \in [0,1]$ yields an density which is 1 on the modes and goes to zero as the distance from the mode increases, $s_\theta(x) :=  1 - \mu_\theta(x)$ provides a measure of how uncertain we are about $x$ being drawn from the training distribution (epistemic uncertainty). Hence $s_\theta(x)$ can be used as an OOD detector.  Figure \ref{figure:two_moons} (bottom row, middle right) shows that $s_\theta(x)$ classifies points as OOD (value ~1) as the distance from the two-moons dataset increases. As Table 1 shows, the unsupervised generative Morse network produces a detector capable of distinguishing MNIST images from the FashionMNIST training images in contrast to other deep generative models \cite{nalisnick2018do}. It is also able to distinguish between CIFAR10 and CIFAR100 when trained on a vision transformer embeddings as reported in Table 1.
Note that for the Gaussian kernel, $V_\theta(x) = -\log \mu_\theta(x)$ coincides with the anomaly score introduced in \cite{ruff2018deeponeclass} and with the Mahalanobis distance when the network is further linear and invertible.

\begin{figure}[th]
\centering
\includegraphics[width=1\columnwidth]{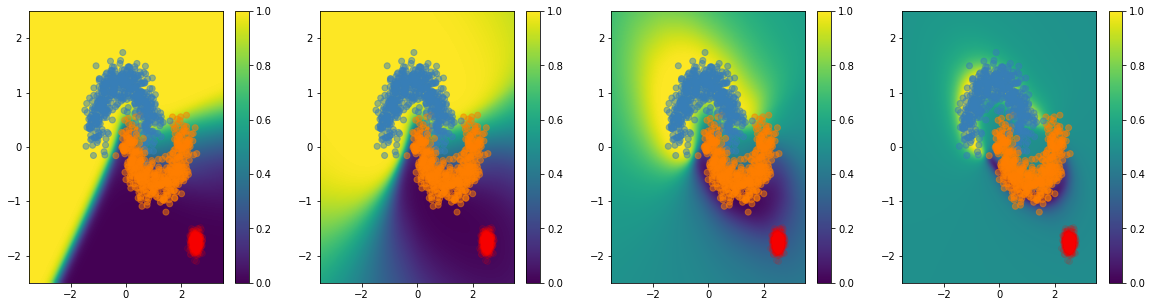}
\includegraphics[width=\columnwidth]{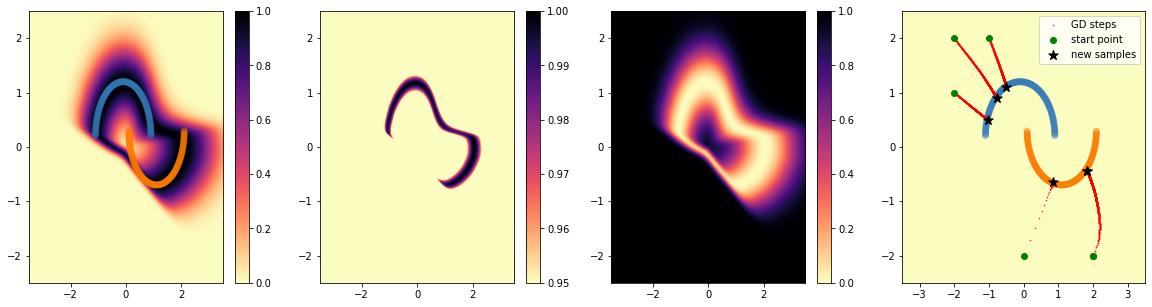}
\vspace{-2em}
\caption{
\textbf{Top row (distance-aware calibration):} \textbf{Left:} Probability output of a ResNet trained to separate the noisy two-moons dataset. \textbf{Middle and Right:} Probability plots of the same ResNet but with its logits calibrated using the unsupervised Gaussian Morse temperature at decreasing kernel bandwidth. The classifier becomes uncertain away from the training data.
\textbf{Bottom row (noiseless two-moons):} \textbf{Left (density plot):} The Morse density concentrates on the two-moons modes. \textbf{Middle Left (mode plot):} The modes learned by the Morse network approximate well the two-moons modes. \textbf{Middle Right (OOD detection):} the Morse detector classifies points away from the modes as OOD. \textbf{Right (sample generation):} Random points following the Morse flow converge to the modes. See Appendix \ref{section:experiment_details} for experiment details.} 
\label{figure:two_moons}
\end{figure}
    
\paragraph{Distance-aware calibration:} We can use the Morse network to calibrate a classifier $f(x) = \textrm{softmax}(h(x))$ trained on a supervised dataset $D=\{(x_i, y_i)\}$ so that it becomes less confident on points far way from the data distribution in the same spirit as SNGP \cite{liu2022simple}. The idea is to fit a Morse network on the input data. Then we can scale the classifier logits $h(x)$ by the inverse of the Morse temperature $T_\theta(x) : = 1/\mu_\theta(x)$. This produces a classifier that becomes uncertain outside of the domain where it has been trained (the more so as the Kernel bandwidth decreases), as demonstrated in Figure \ref{figure:two_moons} (top row). Note that when using the kernel $K_\lambda(z, z') = 1/\sqrt{(1 + \lambda \|z - z'\|^2)}$, the temperature becomes $T_\theta(x) = \sqrt{1 + \lambda V_\theta(x)}$, where $V_\theta(x) = -\log \mu_\theta(x)$, which has the same form as the SNGP temperature of \cite{liu2022simple,liu2020simple}, except that the squared distance function $V_\theta(x)$ replaces the SNGP approximate variance.

\paragraph{Sample generation:} The Morse gradient field $F(x) = -\nabla_x V_\theta(x)$ is attracted to the density modes, which are the global minima (zeros) of the positive function $V_\theta(x) = -\log K(\phi_\theta(x), a)$. If we follow its gradient flow from a random point, we obtain a new sample close to the data distribution mode. This approach resembles the Poisson generative flow from \cite{xu2022poisson}, where the $V_\theta(x)$ is called a potential and is derived by solving the Poisson PDE. We illustrate this approach with the two-moons dataset in Figure \ref{figure:two_moons} (bottom, right) where random points following the Morse flow converge toward the two-moons dataset modes using gradient descent on $V_\theta(X)$.

\begin{table}[h] 
\begin{center} \label{table:ood_detection_results}
\begin{tabular}{c | c | c} 
 IND/OOD & Method & AUROC\\ 
 \hline
 FashionMNIST / MNIST & Morse  & 0.998     \\ 
                      & $\textbf{DoSE}_{KDE}$ & 0.998 \\ 
 CIFAR10 / CIFAR100   & Morse      & 0.955  \\
                      & $\textbf{DoSE}_{SVM}$ & 0.571 \\
\end{tabular}
\caption{\textbf{OOD detection with unsupervised Morse:} As a proof of concept, we report unsupervised Morse AUROC along with best baselines from \cite{warren2021density}. 
For CIFAR10/CIFAR100, we report Morse AUROC trained on embeddings from a pre-trained vision transformer to demonstrate the benefit of this approach. If trained on CIFAR10 directly, Morse AUROC is 0.569, corresponding to the second best performance ($\textbf{DoSE}_{KDE}$) from \cite{warren2021density}. (See Appendix \ref{section:experiment_details}.)
}
\end{center}
\end{table}

\section{Fitting the Morse neural network}
\label{section:fitting}

Consider a dataset $D = \{x_1, \dots, x_n\}$ sampled from a data distribution with density $p(x)$. Theoretically, we want to find the neural network parameters $\theta$ that minimize the KL divergence $\textrm{KL}(p(x) \|\, \mu_\theta(x))$ for unnormalized densities (i.e., positive measures; see \cite{AmariShun-ichi2016IGaI}) between the probability density $p(x)$ generating the sample and the Morse network density $\mu_\theta(x)$, that is,
\begin{equation}
   \mathbb{E}_{x \sim p(x)}\left(\log \frac{p(x)}{\mu_\theta(x)}\right) \\
   + \int \mu_\theta(x) dx - \int p(x) dx
\end{equation}
which amounts to minimizing w.r.t. $\theta$ the following quantity
\begin{equation*}
    \mathbb{E}_{x \sim p(x)}\left(-\log K(\phi_\theta(x), a)\right)
    + \mathbb{E}_{x \sim \mathrm{uni}}\left(K(\phi_\theta(x), a)\right)
\end{equation*}
The corresponding empirical loss is then
\begin{equation*}
    L(\theta) = -\frac 1n\sum_{x\in D} \log K(\phi_\theta(x), a) + \frac 1n \sum_{x \in D_{\textrm{uni}}} K(\phi_\theta(x), a)
\end{equation*}
which can be optimized with any iterative gradient-based optimizer. Note that $D_{\textrm{uni}}$ are points uniformly sampled in $X=\R^d$, and that the second term in the Morse loss can be interpreted as a form of regularization in the spirit of the mixup regularizer term \cite{pinto2022using}. For the Gaussian kernel the first term of this loss (with an added L2 penalty) coincides with the one-class Deep SVDD objective proposed in \cite{ruff2018deeponeclass} for anomaly detection as a way to simplify the loss of a deep one-class SVM in the case of normal data.

\section{The supervised Morse neural network}
\label{section:supervised}

The Morse neural network architecture offers a natural way to incorporate supervised labels to obtain a finer grained generative model. There are two ways to do it:

\paragraph{Separate Morse networks:}
For each label $y\in \{1,\dots,C\}$, we can fit a separate unsupervised Morse network to the subset of the data with label $y=i$ producing a separate density $\mu(x | i) = K_i(\phi_{\theta_i}(x), a_i)$ for each label. The overall density is then given by the average $\mu(x) = \frac 1C \sum_i \mu(x | i)$. The corresponding OOD detector is $s(x) = 1 - \mu(x)$.
We can also create a classifier from this data by interpreting $V_i(x) = -\log \mu(x | i)$ as a squared distance between $x$ and the modes of the data with label $i$; the classifier simply associates to a point $x$ the label $i$ which has the smallest $V_i(x)$. When the kernels are taken to be all Gaussian, the resulting classifier  coincides with the ILOC classifier \cite{sun2021iloc} (except for the regularizing terms in the loss) used in the context of continual learning to avoid catastrophic forgetting since it allows to learn new tasks in complete isolation from the previous ones. (See also \cite{hu2021continual} for similar ideas in the same continual learning context).

\paragraph{Shared Morse network:}
The previous approach can be computationally intensive in the case of a large number of labels. The Morse network offers us a more efficient way to proceed by taking as model for the distribution join density
\begin{equation}
    \mu(x, y) = K(\phi_{\theta}(x), T(y)).
\end{equation}
where $T(y)$ is the one-hot-encoded version of the label $y\in\{1,\dots, C\}$ and $\phi_\theta:X\rightarrow \R^C$ is a shared neural network for all the labels. (For simplicity, we identify $y$ and $T(y)$ and will use at time $e_i$ to denote the $i^{th}$ basis vector.) We can use the same KL divergence  minimization principle (between unnormalized densities) as in the unsupervised case exposed in section \ref{section:fitting} but using now the joint density $p(x,y)$ as the density generating the data, yielding the following empirical Loss for the supervised network:
\begin{equation*}
    L(\theta) = 
    -\frac 1n\sum_{(x,y)} \log K(\phi_\theta(x), y) 
    + \frac 1n \sum_{(x', y')} K(\phi_\theta(x'), y')
\end{equation*}
where $(x,y)$ range over the supervised dataset, and the $(x',y')$'s are obtained by uniform sampling on the join feature and label space. 
In this case, the generative density can be obtained by marginalization
\begin{equation}
    \mu(x) = \sum_{y} K(\phi_\theta(x), y),
\end{equation}
producing the OOD detector $s(x) = 1 - \mu(x)$. 
\paragraph{Experimental results:} Using this supervised Morse detector trained on the embeddings $h(x)$ of a vision transformer fined-tuned to CIFAR10 we observe an improvement in AUROC from 0.955 (for the unsupervised Morse detector) to 0.969 in the same setting for CIFAR100 detection.  Figure \ref{figure:two_moons_supervised} also illustrates visually how the supervised Morse network can learn disconnected mode submanifolds better than the unsupervised version. (See section \ref{section:experiment_details} for details).

Note that we also obtain a classifier that is distance-aware in the sense of \cite{liu2022simple} by construction:
\begin{equation}
    \mu(y | x) 
    = 
    \frac{K(\phi_\theta(x), y)}{\sum_{y'}K(\phi_\theta(x), y')}
    =
    \frac{\exp(-V_y(x))}{\sum_{y'}\exp(-V_{y'}(x))}
\end{equation}
In the case of the Laplace kernel, this classifier coincides with the classifier used for entropic OOD detection \cite{macedo2021entropic} where the entropy value of the classifier output is used as an OOD score. The main difference between the Morse supervised classifier with Laplace kernel and the entropic classifier from \cite{macedo2021entropic} is that the $y$'s are learned for each class and the maximum likelihoods loss is used on $\mu(y|x)$ rather than the Morse loss, making the second term of both loss related but different. 

\section*{Acknowledgements}

We would like to thank Jasper Snoek, Sharat Chikkerur, Mihaela Rosca, James Allingham, Alan Weinstein, and the reviewers for helpful discussions and feedback as well as Patrick Cole his support.


\begin{thebibliography}{45}
\providecommand{\natexlab}[1]{#1}
\providecommand{\url}[1]{\texttt{#1}}
\expandafter\ifx\csname urlstyle\endcsname\relax
  \providecommand{\doi}[1]{doi: #1}\else
  \providecommand{\doi}{doi: \begingroup \urlstyle{rm}\Url}\fi

\bibitem[Amari(2016)]{AmariShun-ichi2016IGaI}
Amari, S.-i.
\newblock \emph{Information Geometry and Its Applications}.
\newblock Applied mathematical sciences. Springer, 2016.

\bibitem[Amodei et~al.(2016)Amodei, Olah, Steinhardt, Christiano, Schulman, and
  Man{\'{e}}]{amodei2016concrete}
Amodei, D., Olah, C., Steinhardt, J., Christiano, P.~F., Schulman, J., and
  Man{\'{e}}, D.
\newblock Concrete problems in {AI} safety.
\newblock \emph{https://arxiv.org/abs/1606.06565}, 2016.

\bibitem[Arbel et~al.(2021)Arbel, Zhou, and Gretton]{arbel2021generalized}
Arbel, M., Zhou, L., and Gretton, A.
\newblock Generalized energy based models.
\newblock In \emph{International Conference on Learning Representations}, 2021.

\bibitem[Austin \& Braam(1995)Austin and Braam]{Austin1995MorseBottTA}
Austin, D.~M. and Braam, P.~J.
\newblock {Morse-Bott} theory and equivariant cohomology.
\newblock 1995.

\bibitem[Banyaga \& Hurtubise(2004)Banyaga and Hurtubise]{banyaga2004aproof}
Banyaga, A. and Hurtubise, D.~E.
\newblock A proof of the {Morse-Bott} lemma.
\newblock \emph{Expositiones Mathematicae}, 2004.

\bibitem[Basu \& Prasad(2021)Basu and Prasad]{basu2020aconnection}
Basu, S. and Prasad, S.
\newblock {A connection between cut locus, Thom space and Morse-Bott
  functions}.
\newblock \emph{to appear in Algebraic \& Geometric Topology.}, 2021.

\bibitem[Bulusu et~al.(2020)Bulusu, Kailkhura, Li, Varshney, and
  Song]{Bulusu2020AnomalousED}
Bulusu, S., Kailkhura, B., Li, B., Varshney, P.~K., and Song, D.~X.
\newblock Anomalous example detection in deep learning: A survey.
\newblock \emph{IEEE Access}, 8, 2020.

\bibitem[Che et~al.(2020)Che, Zhang, Sohl{-}Dickstein, Larochelle, Paull, Cao,
  and Bengio]{che2010yougan}
Che, T., Zhang, R., Sohl{-}Dickstein, J., Larochelle, H., Paull, L., Cao, Y.,
  and Bengio, Y.
\newblock Your {GAN} is secretly an energy-based model and you should use
  discriminator driven latent sampling.
\newblock In Larochelle, H., Ranzato, M., Hadsell, R., Balcan, M., and Lin, H.
  (eds.), \emph{NeurIPS}, 2020.

\bibitem[Gao et~al.(2021)Gao, Song, Poole, Wu, and Kingma]{gao2021learning}
Gao, R., Song, Y., Poole, B., Wu, Y.~N., and Kingma, D.~P.
\newblock Learning energy-based models by diffusion recovery likelihood.
\newblock In \emph{International Conference on Learning Representations}, 2021.

\bibitem[Goodfellow et~al.(2015)Goodfellow, Shlens, and
  Szegedy]{goodfellow2015explaining}
Goodfellow, I., Shlens, J., and Szegedy, C.
\newblock Explaining and harnessing adversarial examples.
\newblock In \emph{ICLR}, 2015.

\bibitem[Goodfellow et~al.(2016)Goodfellow, Bengio, and
  Courville]{goodfellow2016deep}
Goodfellow, I.~J., Bengio, Y., and Courville, A.
\newblock \emph{Deep Learning}.
\newblock MIT Press, 2016.

\bibitem[Guo et~al.(2017)Guo, Pleiss, Sun, and
  Weinberger]{guo2017oncalibration}
Guo, C., Pleiss, G., Sun, Y., and Weinberger, K.~Q.
\newblock On calibration of modern neural networks.
\newblock In \emph{ICML}, 2017.

\bibitem[Han et~al.(2021)Han, Lakshminarayanan, and Liu]{han2021reliable}
Han, K., Lakshminarayanan, B., and Liu, J.~Z.
\newblock Reliable graph neural networks for drug discovery under
  distributional shift.
\newblock In \emph{NeurIPS 2021 Workshop on Distribution Shifts: Connecting
  Methods and Applications}, 2021.

\bibitem[Hendrycks \& Gimpel(2017)Hendrycks and Gimpel]{hendrycks2017a}
Hendrycks, D. and Gimpel, K.
\newblock A baseline for detecting misclassified and out-of-distribution
  examples in neural networks.
\newblock In \emph{International Conference on Learning Representations}, 2017.
\newblock URL \url{https://openreview.net/forum?id=Hkg4TI9xl}.

\bibitem[Hirsch(1997)]{hirsch1997differential}
Hirsch, M.~W.
\newblock \emph{Differential topology}.
\newblock Springer New York, 1997.
\newblock ISBN 0387901485 3540901485.

\bibitem[Hu et~al.(2021)Hu, Qin, Wang, Ma, and Liu]{hu2021continual}
Hu, W., Qin, Q., Wang, M., Ma, J., and Liu, B.
\newblock Continual learning by using information of each class holistically.
\newblock \emph{AAAI}, 2021.

\bibitem[Jumper et~al.(2021)Jumper, Evans, Pritzel, Green, Figurnov,
  Ronneberger, Tunyasuvunakool, Bates, Z{\'i}dek, Potapenko, Bridgland, Meyer,
  Kohl, Ballard, Cowie, Romera-Paredes, Nikolov, Jain, Adler, Back, Petersen,
  Reiman, Clancy, Zielinski, Steinegger, Pacholska, Berghammer, Bodenstein,
  Silver, Vinyals, Senior, Kavukcuoglu, Kohli, and
  Hassabis]{Jumper2021HighlyAP}
Jumper, J.~M., Evans, R., Pritzel, A., Green, T., Figurnov, M., Ronneberger,
  O., Tunyasuvunakool, K., Bates, R., Z{\'i}dek, A., Potapenko, A., Bridgland,
  A., Meyer, C., Kohl, S. A.~A., Ballard, A., Cowie, A., Romera-Paredes, B.,
  Nikolov, S., Jain, R., Adler, J., Back, T., Petersen, S., Reiman, D.~A.,
  Clancy, E., Zielinski, M., Steinegger, M., Pacholska, M., Berghammer, T.,
  Bodenstein, S., Silver, D., Vinyals, O., Senior, A.~W., Kavukcuoglu, K.,
  Kohli, P., and Hassabis, D.
\newblock Highly accurate protein structure prediction with alphafold.
\newblock \emph{Nature}, 596:\penalty0 583 -- 589, 2021.

\bibitem[Kivlichan et~al.(2021)Kivlichan, Liu, Vasserman, and
  Lin]{kivlichan2021measuring}
Kivlichan, I., Liu, J., Vasserman, L.~H., and Lin, Z. (eds.).
\newblock \emph{Measuring and Improving Model-Moderator Collaboration using
  Uncertainty Estimation}, 2021.

\bibitem[Lakshminarayanan et~al.(2017)Lakshminarayanan, Pritzel, and
  Blundell]{lakshminarayanan2017simple}
Lakshminarayanan, B., Pritzel, A., and Blundell, C.
\newblock Simple and scalable predictive uncertainty estimation using deep
  ensembles.
\newblock In \emph{NeurIPS}, 2017.

\bibitem[Larson et~al.()Larson, Mahendran, Peper, Clarke, Lee, Hill,
  Kummerfeld, Leach, Laurenzano, Tang, and Mars]{larson-etal-2019-evaluation}
Larson, S., Mahendran, A., Peper, J.~J., Clarke, C., Lee, A., Hill, P.,
  Kummerfeld, J.~K., Leach, K., Laurenzano, M.~A., Tang, L., and Mars, J.
\newblock An evaluation dataset for intent classification and out-of-scope
  prediction.
\newblock In \emph{EMNLP-IJCNLP 2019}.

\bibitem[Lee et~al.(2018)Lee, Lee, Lee, and Shin]{kimin2018AsimpleUnified}
Lee, K., Lee, K., Lee, H., and Shin, J.
\newblock A simple unified framework for detecting out-of-distribution samples
  and adversarial attacks.
\newblock In \emph{NeurIPS}, volume~31, 2018.

\bibitem[Liu et~al.(2020)Liu, Lin, Padhy, Tran, Bedrax-Weiss, and
  Lakshminarayanan]{liu2020simple}
Liu, J.~Z., Lin, Z., Padhy, S., Tran, D., Bedrax-Weiss, T., and
  Lakshminarayanan, B.
\newblock Simple and principled uncertainty estimation with deterministic deep
  learning via distance awareness.
\newblock In \emph{NeurIPS}, 2020.

\bibitem[Liu et~al.(2022)Liu, Padhy, Ren, Lin, Wen, Jerfel, Nado, Snoek, Tran,
  and Lakshminarayanan]{liu2022simple}
Liu, J.~Z., Padhy, S., Ren, J., Lin, Z., Wen, Y., Jerfel, G., Nado, Z., Snoek,
  J., Tran, D., and Lakshminarayanan, B.
\newblock A simple approach to improve single-model deep uncertainty via
  distance-awareness.
\newblock \emph{JMLR}, 2022.

\bibitem[Mac{\^{e}}do et~al.(2021)Mac{\^{e}}do, Ren, Zanchettin, Oliveira, and
  Ludermir]{macedo2021entropic}
Mac{\^{e}}do, D., Ren, T.~I., Zanchettin, C., Oliveira, A. L.~I., and Ludermir,
  T.~B.
\newblock Entropic out-of-distribution detection.
\newblock In \emph{International Joint Conference on Neural Networks (IJCNN)},
  pp.\  1--8, 2021.

\bibitem[Morningstar et~al.(2021)Morningstar, Ham, Gallagher, Lakshminarayanan,
  Alemi, and Dillon]{warren2021density}
Morningstar, W.~R., Ham, C., Gallagher, A.~G., Lakshminarayanan, B., Alemi, A.,
  and Dillon, J.~V.
\newblock Density of states estimation for out of distribution detection.
\newblock In \emph{AISTATS}, 2021.

\bibitem[Morse(1939)]{Morse1939TheBO}
Morse, A.~P.
\newblock The behavior of a function on its critical set.
\newblock \emph{Annals of Mathematics}, 40:\penalty0 62--70, 1939.

\bibitem[Murphy(2023)]{pml2Book}
Murphy, K.~P.
\newblock \emph{Probabilistic Machine Learning: Advanced Topics}.
\newblock MIT Press, 2023.
\newblock URL \url{http://probml.github.io/book2}.

\bibitem[Nagarajan et~al.(2021)Nagarajan, Andreassen, and
  Neyshabur]{nagarajan2021understanding}
Nagarajan, V., Andreassen, A., and Neyshabur, B.
\newblock Understanding the failure modes of out-of-distribution
  generalization.
\newblock In \emph{ICLR}, 2021.

\bibitem[Nalisnick et~al.(2019)Nalisnick, Matsukawa, Teh, Gorur, and
  Lakshminarayanan]{nalisnick2018do}
Nalisnick, E., Matsukawa, A., Teh, Y.~W., Gorur, D., and Lakshminarayanan, B.
\newblock Do deep generative models know what they don't know?
\newblock In \emph{ICLR}, 2019.

\bibitem[Nguyen et~al.(2015)Nguyen, Yosinski, and
  Clune]{nguyen2014easilyfooled}
Nguyen, A.~M., Yosinski, J., and Clune, J.
\newblock Deep neural networks are easily fooled: High confidence predictions
  for unrecognizable images.
\newblock In \emph{CVPR}, 2015.

\bibitem[Ovadia et~al.(2019)Ovadia, Fertig, Ren, Nado, Sculley, Nowozin,
  Dillon, Lakshminarayanan, and Snoek]{ovadia2019can}
Ovadia, Y., Fertig, E., Ren, J., Nado, Z., Sculley, D., Nowozin, S., Dillon,
  J., Lakshminarayanan, B., and Snoek, J.
\newblock Can you trust your model's uncertainty? evaluating predictive
  uncertainty under dataset shift.
\newblock In \emph{NeurIPS}, 2019.

\bibitem[Pinto et~al.(2022)Pinto, Yang, Lim, Torr, and Dokania]{pinto2022using}
Pinto, F., Yang, H., Lim, S.-N., Torr, P., and Dokania, P.~K.
\newblock Using mixup as a regularizer can surprisingly improve accuracy \&
  out-of-distribution robustness.
\newblock In \emph{NeurIPS}, 2022.

\bibitem[Prasad(2023)]{prasad2023cutlocus}
Prasad, S.
\newblock Cut locus of submanifolds: A geometric and topological viewpoint.
\newblock 2023.

\bibitem[Ren et~al.(2019)Ren, Liu, Fertig, Snoek, Poplin, Depristo, Dillon, and
  Lakshminarayanan]{ren2019likelihood}
Ren, J., Liu, P.~J., Fertig, E., Snoek, J., Poplin, R., Depristo, M., Dillon,
  J., and Lakshminarayanan, B.
\newblock Likelihood ratios for out-of-distribution detection.
\newblock In \emph{NeurIPS}, volume~32, 2019.

\bibitem[Ren et~al.(2021)Ren, Fort, Liu, Roy, Padhy, and
  Lakshminarayanan]{ren2021SimpleFix}
Ren, J., Fort, S., Liu, J.~Z., Roy, A.~G., Padhy, S., and Lakshminarayanan, B.
\newblock A simple fix to {M}ahalanobis distance for improving near-ood
  detection.
\newblock \emph{arXiv preprint arXiv:2106.09022}, 2021.

\bibitem[Roy et~al.(2021)Roy, Ren, Azizi, Loh, Natarajan, Mustafa, Pawlowski,
  Freyberg, Liu, Beaver, Vo, Bui, Winter, MacWilliams, Corrado, Telang, Liu,
  Cemgil, Karthikesalingam, Lakshminarayanan, and
  Winkens]{Roy2021doesyourdermatology}
Roy, A.~G., Ren, J.~J., Azizi, S., Loh, A., Natarajan, V., Mustafa, B.,
  Pawlowski, N., Freyberg, J., Liu, Y., Beaver, Z.~W., Vo, N., Bui, P., Winter,
  S., MacWilliams, P., Corrado, G., Telang, U., Liu, Y., Cemgil, T.,
  Karthikesalingam, A., Lakshminarayanan, B., and Winkens, J.
\newblock Does your dermatology classifier know what it doesn't know? detecting
  the long-tail of unseen conditions.
\newblock \emph{Medical Imaging Analysis}, 2021.

\bibitem[Ruff et~al.(2018)Ruff, Vandermeulen, Goernitz, Deecke, Siddiqui,
  Binder, M{\"u}ller, and Kloft]{ruff2018deeponeclass}
Ruff, L., Vandermeulen, R., Goernitz, N., Deecke, L., Siddiqui, S.~A., Binder,
  A., M{\"u}ller, E., and Kloft, M.
\newblock Deep one-class classification.
\newblock In \emph{ICML}, 2018.

\bibitem[Salehi et~al.(2022)Salehi, Mirzaei, Hendrycks, Li, Rohban, and
  Sabokrou]{salehi2022a}
Salehi, M., Mirzaei, H., Hendrycks, D., Li, Y., Rohban, M.~H., and Sabokrou, M.
\newblock A unified survey on anomaly, novelty, open-set, and out
  of-distribution detection: Solutions and future challenges.
\newblock \emph{TMLR}, 2022.

\bibitem[Sard(1942)]{Sard1942TheMO}
Sard, A.
\newblock The measure of the critical values of differentiable maps.
\newblock \emph{Bulletin of the American Mathematical Society}, 48:\penalty0
  883--890, 1942.

\bibitem[Sharma et~al.(2022)Sharma, Debaque, Duclos, Chehri, Kinder, and
  Fortier]{sharma2022deep}
Sharma, T., Debaque, B., Duclos, N., Chehri, A., Kinder, B., and Fortier, P.
\newblock Deep learning-based object detection and scene perception under bad
  weather conditions.
\newblock \emph{Electronics}, 11\penalty0 (4), 2022.

\bibitem[Sun et~al.(2021)Sun, Zhang, Wang, Geng, and Li]{sun2021iloc}
Sun, W., Zhang, J., Wang, D., Geng, Y.-a., and Li, Q.
\newblock {ILCOC}: An incremental learning framework based on contrastive
  one-class classifiers.
\newblock In \emph{CVPR}, 2021.

\bibitem[Wang et~al.(2021)Wang, Li, Che, Zhou, Liu, and
  Li]{wang2021energybased}
Wang, Y., Li, B., Che, T., Zhou, K., Liu, Z., and Li, D.
\newblock Energy-based open-world uncertainty modeling for confidence
  calibration.
\newblock In \emph{ICCV 2021}, 2021.

\bibitem[Wu et~al.(2021)Wu, Esmaeili, Wick, Tristan, and van~de
  Meent]{wu2021conjugate}
Wu, H., Esmaeili, B., Wick, M.~L., Tristan, J.-B., and van~de Meent, J.-W.
\newblock Conjugate energy-based models.
\newblock In \emph{Third Symposium on Advances in Approximate Bayesian
  Inference}, 2021.

\bibitem[Xu et~al.(2022)Xu, Liu, Tegmark, and Jaakkola]{xu2022poisson}
Xu, Y., Liu, Z., Tegmark, M., and Jaakkola, T.~S.
\newblock Poisson flow generative models.
\newblock In \emph{NeurIPS}, 2022.

\bibitem[Zhao et~al.(2017)Zhao, Mathieu, and LeCun]{zhao2017energybased}
Zhao, J., Mathieu, M., and LeCun, Y.
\newblock Energy-based generative adversarial networks.
\newblock In \emph{International Conference on Learning Representations}, 2017.

\end{thebibliography}

\newpage
\appendix
\onecolumn

\section{Morse-Bott condition and squared distances}\label{section:morse}

An important notion in uncertainty quantification is the notion of a squared distance from the training set. It can be used to devise anomaly scores, OOD detectors, or distance-aware input-dependent calibration temperatures. For instance the SNGP approximate variance from \cite{liu2020simple, liu2022simple} can be thought of as such a squared distance. Similarly the Mahalanobis distance is also a form of a squared distance from a point to the single mode of an underlying multivariate Gaussian distribution \cite{kimin2018AsimpleUnified, ren2021SimpleFix}. For densities that have modes that are no longer discrete sets but rather submanifolds $M\subset X$ in the feature (or embedding) space $X$, the corresponding notion is that of a minimal distance between a point $x\in X$ and the mode submanifold $M$:
$$
d(x) = \min_{m\in M} \{d(x,m):\: m\in M\}.
$$
In \cite{basu2020aconnection,prasad2023cutlocus}, the authors show that such a function  encodes important information about the topology of the submanifold $M$. In particular, they show that the square $V(x)=d(x)^2$ of a distance function $d(x)$ from a point $x$ in a Riemannian submanifold to submanifold $M\subset X$ satisfies the Morse-Bott non-degeneracy condition on $M$ \cite{Austin1995MorseBottTA}. 
Recall that a function (with a submanifold of critical points) satisfies the Morse-Bott non-degeneracy condition at a critical point if its Hessian vanishes only in the directions tangent to the critical submanifold at that point. 
The Morse-Bott condition is important for squared distances because when a positive function $V(x)$ satisfies it on its set of global minima $M=\{x\in X:\:V(x) = 0\}$  the Morse-Bott Lemma \cite{banyaga2004aproof} tells us that $V(x)$  can be expressed locally as squared distance from $M$ for a certain metric.

The theorem below shows that, for a kernels $K$ satisfying the conditions stated in Definition \ref{definition:morse_kernel} (we call these kernels \emph{Morse kernels}), the negative logarithm of a Morse network
$$
V_\theta(x) = -\log K(\phi_\theta(x), a)
$$
satisfies the Morse-Bott condition on its submanifold of zeros. Since this submanifold is the mode submanifold $M$ of the Morse density
$
\mu_\theta(x) = K(\phi_\theta(x), a),
$
and since $V$ is positive away from $M$ for Morse kernels, the Morse-Bott lemma tells us that $V(x)$ is locally a squared distance from the mode submanifold for a certain metric. (The characterization of this metric is beyond the scope of this work.) 
This provides an intuitive formulation of the Morse network density as the negative exponential of the distance from the mode submanifold  
$$\mu_\theta(x) = e^{-V_\theta(x)}.$$
Note that this formulation resembles an unnormalized Gibbs measure where $V_\theta(x)$ plays the role of the configuration energy with minimal energy located at the density modes. 

We now give a precise definition of what a Morse kernel is before proving our theorem. 

\begin{definition}\label{definition:morse_kernel} 
A \emph{Morse kernel} $K$ on a space $Z=\R^k$ is a positive (non-necessarily symmetric) kernel $K(z_1, z_2)$ taking its values in the interval $[0,1]$ and such that $K(z_1,z_2) = 1$ if and only if $z_1 = z_2$.
\end{definition}

Note that all kernels of the form $K(z_1, z_2) = \exp(-\lambda D(z_1, z_2))$ where $D$ is a divergence in the sense of Amari \cite{AmariShun-ichi2016IGaI} are Morse kernels (since $D(z_1, z_2) = 0$ if and only if $z_1 = z_2$, in particular, the Gaussian kernel is a Morse Kernel. The Cauchy kernel  $K(z_1, z_2) = \frac 1{1 + \lambda \|z_1 - z_2\|^2}$ is also an example of a Morse kernel.

\begin{theorem} \label{thm:morse}
Let $\mu_\theta(x) = K(\phi_\theta(x), a)$ be a Morse network with Morse kernel $K$. We denote by $M_a = \{x:\:\phi_\theta(x) = a\}$ the level set of $\phi_\theta(x)$ at level $a$. Then we have the following properties:
\begin{enumerate}
    \item \label{statement1} The Morse network $\mu_\theta(x)$ takes its values in $[0,1]$.
    \item \label{statement2} $M_a$ is the mode submanifold of $\mu_\theta(x)$, that is, the locus of points in $X$ where $\mu_\theta(x)$ reaches its highest value $1$.
    \item \label{statement3} The function $V_\theta(x) = -\log \mu_\theta(x)$ is positive with its locus of zeros coinciding with $M_a$. 
    
    \item \label{statement4} $V_\theta(x)$ satisfies the Morse-Bott condition on $M_a$.
\end{enumerate}
\end{theorem}

\begin{proof}
Statement \ref{statement1} derives from the fact that a Morse kernel takes its values in $[0, 1]$ by definition.

Statement \ref{statement2} comes from that $\mu_\theta(x) = 1$ happens if and only if $K(\phi_\theta(x), a) = 1$, which happens if and only if $\phi_\theta(x) = a$ (i.e., if and only if $x\in M_a$) for a Morse kernel.

Statement \ref{statement3} is obvious: The negative logarithm is positive on $[0, 1]$ and vanishes when $\mu_\theta(x) = 1$, that is if and only if $x$ is a point in the mode submanifold $M_a$.

Statement \ref{statement4} requires us to show that its Hessian $\nabla^2 V_\theta(x)$ vanishes only on the tangent space to $M_a$ at all points $x$ in $M_a$. 
For the sake of simplicity, let us prove this only in the case where $\phi_\theta(x)$ takes its values in $Z=\R$. To begin with, let us compute the gradient of $V_\theta(x)$ (we remove in the computation below the dependence on $\theta$ to simplify the notation):
\begin{equation}
    \nabla_x V(x) = \frac {-1}{K(\phi(x), a)}\frac{\partial K}{\partial z_1}(\phi(x), a)\nabla_x\phi(x)
\end{equation}
Then Hessian has then three terms:
\begin{equation*}
\nabla_x^2 V(x) 
= 
    \frac{\partial}{\partial x}\left(\frac {-1}K\right)\frac{\partial K}{\partial z_1}(\phi(x), a)\nabla_x\phi(x)
    + \frac{-1}K \frac{\partial^2 K}{\partial z_1\partial z_1}K(\phi(x), a) \nabla_x\phi(x)^T\nabla_x\phi(x) 
    + \frac{-1}K \frac{\partial K}{\partial z_1}(\phi(x), a)\nabla_x^2\phi(x)
\end{equation*}

Now by definition for $x\in M_a$, we have that $\phi(x) = a$. Since for each $a'\in Z$ the function $a'\rightarrow K(a', a)$ reaches its global maximal value 1 at $a'= a$, we have that , we have that:
\begin{equation}
    K(a, a) = 1, \quad \frac{\partial K}{\partial z_1}(a, a) = 0, \quad \textrm{and}\quad
    \frac{\partial^2 K}{\partial z_1\partial z_1}(a, a) < 0.
\end{equation}
Using this last relations when evaluating the Hessian of $V$ at a point $x \in M_a$, we obtain that all the terms except for the middle one in the Hessian expression vanish, yielding:
\begin{equation}
    \nabla_x^2 V(x)
= 
    - \frac{\partial^2 K}{\partial z_1\partial z_1}(a, a)\nabla_x\phi(x)^T \nabla_x\phi(x)
=   - \frac{\partial^2 K}{\partial z_1\partial z_1}(a, a)P(x),
\end{equation}
where $P(x) = \nabla_x\phi(x)^T \nabla_x\phi(x)$ is the orthogonal projector onto the subspace spanned by the gradient $\nabla_x \phi_\theta(x)$ for all $x\in M_a$.
Thus,  $\nabla_x^2 V(x) v = 0$ happens if and only if $P(x) v = 0$, which is equivalent to $v$ being orthogonal to the gradient $\nabla_x\phi_\theta(x)$. In other words, $\nabla_x^2 V(x) v = 0$ if and only if $v$ is in the tangent space to $M_a$ at x (since the $\nabla_x\phi_\theta(x)$ is orthogonal to that space). 
\end{proof}

\begin{remark}
In the theorem above, the hyper-parameter $a$ needs to be a regular value of the map $\phi_\theta(x)$ (i.e. $a$ should not be in the image of the critical points of $\phi_\theta(x)$) in order for $M_a$ to be a submanifold (see \cite{hirsch1997differential} for details). Luckily,  \citet{Morse1939TheBO} and  \citet{Sard1942TheMO} showed that the values $a$ for which this is not the case (i.e., the image of the critical values) has measure null in $Z$.
\end{remark}

\begin{remark}
Let us conclude this section by observing that 
\citet{basu2020aconnection} and \citet{prasad2023cutlocus} show that the square of the distance function $d(x)$ from a point $x$ in a Riemannian submanifold $X$ to submanifold $M\subset X$ encodes important topological information about the submanifold $M$. From that observation, we conjecture that the negative logarithm of a fitted Morse network $V_\theta(x) = -\log \mu_\theta(x)$ also encodes important topological information about the mode submanifold of the data distribution. We refer the reader to \cite{Austin1995MorseBottTA} for a detailed presentation of Morse-Bott theory in topology. 
\end{remark}

\section{Experiment details.}
\label{section:experiment_details}

\paragraph{Figure 1, top row:}
We trained a ResNet with 6 ResNet blocks consisting in a dense layer with 128 neurons followed by a dropout layer with dropout rate set to 0.1 and a skip connection to classify the \href{https://scikit-learn.org/stable/modules/generated/sklearn.datasets.make\_moons.html}{two-moons dataset} with the noise parameter set to 0.2. The data was fitted with 100 epochs using Adam and a learning rate of 1e-4 on a batch size of 128. The network outputs a logit vector with two components, one for each class.  The decision boundary found by this model is plotted on the leftmost plot on the top row. We see that the model is very confident once one moves away from the decision boundary. The next 3 figures on this row use the same model, except that their logits are scaled  by the Morse temperature $T(x) = 1/\mu_\theta(x)$ with decreasing kernel bandwidths : 1, 0.1, and 0.01. One observes that the scaled ResNet becomes more uncertain away from the training data (rather than the decision boundary) and the more so as the kernel bandwidth decreases. The Morse network whose temperature was used to scale the ResNet logits was trained independently directly on the two-moons data in a fully unsupervised fashion. For the Morse network we used 5 dense ReLU layers each with 500 neurons followed by a ReLU output layer with one neuron (target space dimension set to 1) and we chose the Gaussian kernel. The hyper-parameters $a$ was set to 2. We trained for 2 epochs with full batch  Adam and learning rate 1e-3.

\paragraph{Figure 1, bottom row:}

We trained an unsupervised Morse model on the noiseless \href{https://scikit-learn.org/stable/modules/generated/sklearn.datasets.make\_moons.html}{two-moons dataset}. The Morse network had a Gaussian kernel, 4 dense ReLU layers with 500 neurons each and a 1-dimensional output layer. The hyper-parameter $a$ was set to 2. We fitted the network with full-batch Adam with learning rate set to 1e-3.

\textbf{Leftmost plot:} We plot the values of the Morse network $\mu_\theta(x)\in[0,1]$ on a \href{https://matplotlib.org/3.5.3/api/_as_gen/matplotlib.pyplot.magma.html}{magma colormap}. We observe that the network has values close to 1 around the data modes (the two moons) although it experiences some difficulty with the leftmost edge of the orange moon. We believe that the reason for this phenomenon is that the mode submanifold is disconnected and the Morse network is trying to learn a connected submanifold. In Figure \ref{figure:two_moons_supervised} below, we show that the supervised Morse network is able to learn disconnected mode submanifolds with one connected piece per label.
\textbf{Middle left plot:} We plot the values of the Morse network restricted between 0.95 and 1, where they most strongly concentrate around the mode lines. We see that the unsupervised Morse network modes approximate well the two-moons with a connected submanifold. \textbf{Middle right plot:} We plot the Morse OOD detector associated with the Morse network $s_\theta(x) = 1 - \mu_\theta(x)$. We see that the detector outputs a probability superior to 0.5 (and rapidly reaches 1) as we depart from the modes. 
\textbf{Rightmost plot:} To visualize how the Morse sample generator works we plot the flow lines of the differential equation $\dot x = -\nabla_x V_\theta(x)$, where $V_\theta(x)=-\log \mu_\theta(x)$ from a number of initial points: 
$
    [0.0, -2.0],
    [-2.0, 2.0],
    [2.0, -2.0],
    [-2.0, 1.0],
    [-1.0, 2.0],
    [2.0, -2.0],
$
To approximate these flow lines we use gradient descent $x_{n+1} = x_n - h\nabla_x V_\theta(x)$ with a learning rate $h=0.001$ for 1000 steps. The last step of this procedure is the generated sample. We see that the flow lines converge rapidly to the learned modes, which is to be expected as $V_\theta$ is a form of squared distance from the learned modes.

\begin{figure}[h]
\centering
\includegraphics[width=0.3\columnwidth]{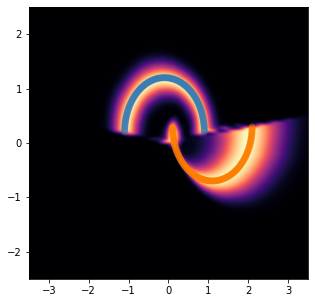}
\caption{
    \textbf{Supervised Morse model:} The supervised Morse model is able to learn disconnected mode submanifolds for each of the labels. The model architecture and training setup is the same as the unsupervised case described above, except for the output layer that has a dimension of 2 (one dimension per label).
} 
\label{figure:two_moons_supervised}
\end{figure}

\paragraph{Table 1:}
\textbf{FashionMNIST:} 
We trained an unsupervised Morse network $\mu_\theta(x)$ directly on the FashionMNIST images, which we flatten into 1-dimensional vectors. The Morse network processed these images through 5 dense ReLU layers with 500 neurons each followed by an 1-dimensional output dense ReLU layer. We used the Gaussian kernel $K_\lambda(z, z') = \exp(-\lambda \|z - z'\|^2)$ with $\lambda = 1$ and the hyper-parameter $a$ was set to 10. For the regularization term in the Morse loss, we sampled uniformly over the hyper-cube with side $[-5, 5]$. To fit the Morse network we used Adam with a learning rate of 0.001 and a batch size of 1000 images over 4 epochs. We computed AUROC of the Morse OOD detector $s_\theta(x) = 1 - \mu_\theta(x)$ for MNIST. To compare our results with a baseline we used the best performance of the unsupervised OOD detectors in the same setting presented in \cite{warren2021density}. The results are in Table 1. 
\textbf{CIFAR10:} 
We trained a Morse network on the embeddings (penultimate layer) produced by a \href{https://github.com/google-research/big_vision}{vision transformer}, which was pre-trained on ImageNet-21k and fine-tuned to classify the CIFAR10 images. We used 12 layers and 12 attention heads with no dropout. The patch size was set to $[16,16]$, the hidden layers dimension to 768, the mlp dimension to  3072, with no dropout. We fine-tuned the pre-trained vision transformer using momentum with learning rate 0.003, batch size of 512 for 10000 steps, gradient clipping at 1.0, with a cosine decay schedule with 500 warm up steps.
We tuned the Morse network for best AUROC on SVHN, which was achieved after 14 steps: the tuned Morse network processed the vision transformer CIFAR10 embeddings through 5 dense leaky ReLU layers with 500 neurons each followed by an 1-dimensional output dense leaky ReLU layer. We used the kernel $K_\lambda (z, z') = \frac{1}{\sqrt{1 + \lambda \|z - z'\|^2}}$ with $\lambda = 0.1$ and the hyper-parameter $a$ was set to 10. For the regularization term in the Morse loss we sampled uniformly over the hyper-cube with side $[-5, 5]$. To fit the Morse network we used Adam with a learning rate of 0.001 and a batch size of 1000 images. We then computed the AUROC of the tuned Morse OOD detector $s_\theta(x) = 1 - \mu_\theta(x)$ for CIFAR100. The results are in Table 1. In the caption of Table 1, we report unsupervised Morse AUROC trained on the CIFAR10 images directly rather than on the pre-trained embeddings of a vision transformer. In that experiment, we re-used the exact same setup for the Morse network as in the case of FashionMNIST, described above.

\paragraph{Supervised Morse, section 4:} 
We trained a supervised Morse network on the same CIFAR10 embeddings produced by the vision transformer described in the unsupervised Morse experiment (see paragraph above) but with the additional information of the image labels. The tuning was done to achieve best AUROC performance on SVHN, which was reached after 9 steps. The tuned Morse network used a Cauchy kernel $K_\lambda (z, z') = \frac{1}{1 + \lambda \|z - z'\|^2}$ with $\lambda = 1$. The hyper-parameter $a$ was set to 1. We used 3 dense leaky ReLU layers with 400 neurons without biases followed by a 10-dimensional (1 dimension per class) output leaky ReLU layer with bias. The uniform sampling for the regularization term in the Morse loss was done in a hyper-cube of side $[-3, 3]$. We trained with Adam and a learning rate of 0.003 with batch size 1000. We then computed the best AUROC for the tuned Morse network on CIFAR100.


\end{document}